\pdfoutput=1

\documentclass{article}

\usepackage{microtype}
\usepackage{graphicx}
\usepackage{subfigure}
\usepackage{booktabs} 

\usepackage{hyperref}

\usepackage{amsmath,amsfonts,bm}

\def\eqref#1{equation~\ref{#1}}

\def\1{\bm{1}}

\DeclareMathAlphabet{\mathsfit}{\encodingdefault}{\sfdefault}{m}{sl}
\SetMathAlphabet{\mathsfit}{bold}{\encodingdefault}{\sfdefault}{bx}{n}

\newcommand{\sigmoid}{\sigma}

\usepackage{booktabs}
\usepackage{hyperref}
\usepackage{url}
\usepackage{graphicx}
\usepackage{subfigure}
\usepackage{amsthm}
\usepackage{comment}
\newtheorem{theorem}{Theorem}

\newcommand{\modelname}{{NECST}}

\usepackage[accepted]{icml2019}

\icmltitlerunning{Neural Joint Source-Channel Coding}

\begin{document}

\twocolumn[
\icmltitle{Neural Joint Source-Channel Coding}




\icmlsetsymbol{equal}{*}

\begin{icmlauthorlist}
\icmlauthor{Kristy Choi}{cs}
\icmlauthor{Kedar Tatwawadi}{ee}
\icmlauthor{Aditya Grover}{cs}
\icmlauthor{Tsachy Weissman}{ee}
\icmlauthor{Stefano Ermon}{cs}
\end{icmlauthorlist}

\icmlaffiliation{cs}{Department of Computer Science, Stanford University}
\icmlaffiliation{ee}{Department of Electrical Engineering, Stanford University}

\icmlcorrespondingauthor{Kristy Choi}{kechoi@cs.stanford.edu}

\icmlkeywords{Machine Learning, ICML}

\vskip 0.3in
]



\printAffiliationsAndNotice{}  

\begin{abstract}
For reliable transmission across a noisy communication channel, classical results from information theory show that it is asymptotically optimal to separate out the source and channel coding processes. However, this decomposition can fall short in the finite bit-length regime, as it requires non-trivial tuning of hand-crafted codes and assumes infinite computational power for decoding. In this work, we propose to jointly learn the encoding and decoding processes using a new discrete variational autoencoder model. By adding noise into the latent codes to simulate the channel during training, we learn to both compress and error-correct given a fixed bit-length and computational budget. We obtain codes that are not only competitive against several separation schemes, but also learn useful robust representations of the data for downstream tasks such as classification. Finally, inference amortization yields an extremely fast neural decoder, almost an order of magnitude faster compared to standard decoding methods based on iterative belief propagation. 
\end{abstract}

\section{Introduction}
We consider the problem of encoding images as bit-strings so that they can be reliably transmitted across a noisy communication channel. Classical results from \citet{shannon1948mathematical} show that for a memoryless communication channel, as the image size goes to infinity, it is optimal to separately: 1) compress the images as much as possible to remove redundant information (source coding) and 2) use an error-correcting code to re-introduce redundancy, which allows for reconstruction in the presence of noise (channel coding). This \emph{separation theorem} has been studied in a wide variety of contexts and has also enjoyed great success in the development of new coding algorithms.

However, this elegant decomposition suffers from two critical limitations in the finite bit-length regime. First, without an infinite number of bits for transmission, the overall distortion (reconstruction quality) is a function of both source and channel coding errors. Thus optimizing for both: (1) the number of bits to allocate to each process and (2) the code designs themselves is a difficult task. Second, maximum-likelihood decoding is in general NP-hard (\citet{berlekamp1978inherent}). Thus obtaining the accuracy that is possible \emph{in theory} is contingent on the availability of infinite computational power for decoding, which is impractical for real-world systems. Though several lines of work have explored various relaxations of this problem to achieve tractability (\citet{koetter2003graph}, \citet{feldman2005using},  \citet{vontobel2007low}), many decoding systems rely on heuristics and early stopping of iterative approaches for scalability, which can be highly suboptimal.
 
To address these challenges we propose Neural Error Correcting and Source Trimming (\modelname) codes, a deep learning framework for jointly learning to compress and error-correct an input image given a fixed bit-length budget. Three key steps are required. First, we use neural networks to encode each image into a suitable bit-string representation, sidestepping the need to rely on hand-designed coding schemes that require additional tuning for good performance. Second, we simulate a discrete channel \textit{within} the model and inject noise directly into the latent codes to enforce robustness. Third, we amortize the decoding process such that after training, we obtain an extremely fast decoder at test time.
These components pose significant optimization challenges due to the inherent non-differentiability of discrete latent random variables. We overcome this issue by leveraging recent advances in unbiased low-variance gradient estimation for variational learning of discrete latent variable models, and train the model using a variational lower bound on the mutual information between the images and their binary representations to obtain good, \emph{robust} codes. At its core, \modelname{}  can also be seen as an implicit deep generative model of the data derived from the perspective of the joint-source channel coding problem (\citet{bengio2014deep}, \citet{sohl2015deep}, \citet{goyal2017variational}). 

In experiments, we test \modelname{} on several grayscale and RGB image datasets, obtaining improvements over industry-standard compression (e.g, WebP \citep{webp}) and error-correcting codes (e.g., low density parity check codes). We also learn an extremely fast neural decoder, yielding almost an order of magnitude (two orders for magnitude for GPU) in speedup compared to standard decoding methods based on iterative belief propagation (\citet{fossorier1999reduced}, \citet{chen2002near}). Additionally, we show that in the process we learn discrete representations of the input images that are useful for downstream tasks such as classification. 

In summary, our contributions are: 1) a novel framework for the JSCC problem and its connection to generative modeling with discrete latent variables; 2) a method for robust representation learning with latent additive discrete noise.
\begin{figure}[h]
    \centering
    \includegraphics[width=0.75\linewidth]{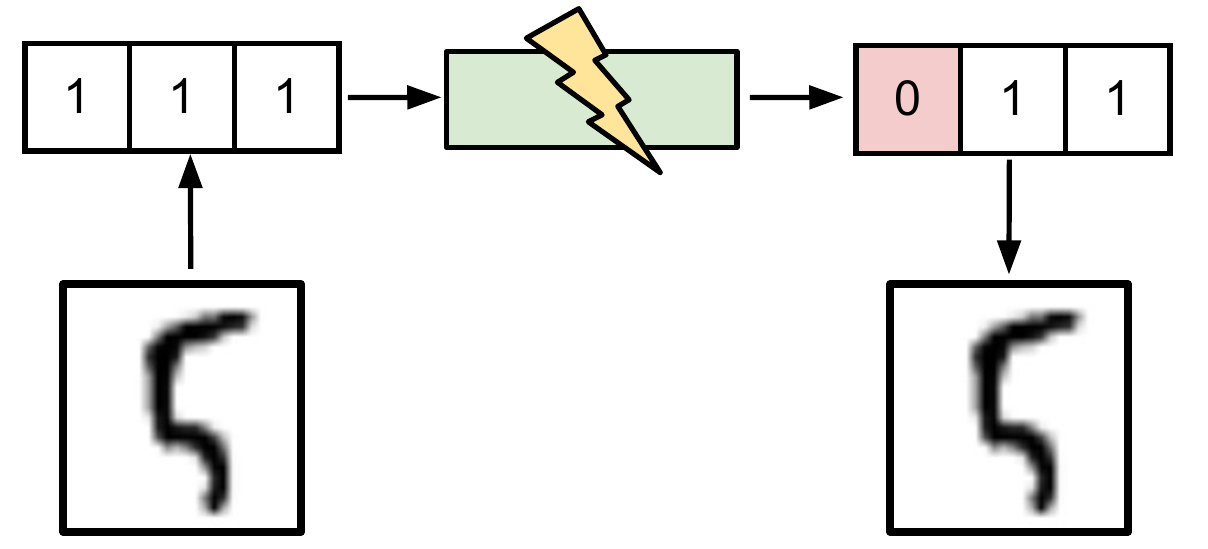}
    \caption{Illustration of \modelname{}. An image is encoded into a binary bit-string, where it is corrupted by a discrete channel that is simulated \textit{within} the model. Injecting noise into the latent space encourages \modelname{} to learn robust representations and also turns it into an implicit generative model of the data distribution.}
    \label{fig:intuition}
\end{figure}
\section{Source and Channel Coding}

\subsection{Preliminaries}
We consider the problem of reliably communicating data across a noisy channel using a system that can detect and correct errors. Let $\cal{X}$ be a space of possible inputs (e.g., images), and $p_{\textrm{data}}(x)$ be a \emph{source} distribution defined over $x \in \mathcal{X}$. The goal of the communication system is to encode samples $x \sim p_{\textrm{data}}(x)$ as messages in a codeword space $\widehat{\mathcal{Y}} = \{0,1\}^m$. The messages are transmitted over a noisy channel, where they are corrupted to become \emph{noisy} codewords in $\mathcal{Y}$. Finally, a decoder produces a reconstruction $\hat{x} \in \cal{X}$  of the original input $x$ from a received \emph{noisy} code $y$. 

The goal is to minimize the overall \emph{distortion} (reconstruction error) $\Vert x-\hat{x} \Vert$ in $\ell_1$ or $\ell_2$ norm while keeping the message length $m$ short (\emph{rate}). In the absence of channel noise, low reconstruction errors can be achieved using a compression scheme to encode inputs $x \sim p_{\textrm{data}}(x)$ as succinctly as possible (e.g., WebP/JPEG). In the presence of noise, longer messages are typically needed to redundantly encode the information and recover from errors (e.g., parity check bits). 

\subsection{The Joint Source-Channel Coding Problem}
Based on \citet{shannon1948mathematical}'s fundamental results, existing systems adhere to the following pipeline. A \textit{source encoder} compresses the source image into a bit-string with the minimum number of bits possible. A \textit{channel encoder} re-introduces redundancies, part of which were lost during source coding to prepare the code $\hat{y}$ for transmission. The \textit{decoder} then leverages the channel code to infer the original signal $\hat{y}$, producing an approximate reconstruction $\hat{x}$.

In the separation theorem, Shannon proved that the above scheme is optimal in the limit of infinitely long messages. That is, we can minimize distortion by optimizing the source and channel coding processes independently. However, given a finite bit-length budget, the relationship between rate and distortion becomes more complex. The more bits that we reserve for compression, the fewer bits we have remaining to construct the best error-correcting code and vice versa. Balancing the two sources of distortion through the optimal bit allocation, in addition to designing the best source and channel codes, makes this \emph{joint source-channel coding} (JSCC) problem challenging.

In addition to these design challenges, real-world communication systems face computational and memory constraints that hinder the straightforward application of Shannon's results. Specifically, practical decoding algorithms rely on approximations that yield suboptimal reconstruction accuracies. The information theory community has studied this problem extensively, and proposed different bounds for finite bit-length JSCC in a wide variety of contexts (\citet{pilc1967coding}, \citet{csiszar1982linear}, \citet{kostina2013lossy}). Additionally, the issue of packet loss and queuing delay in wireless video communications suggest that modeling the forward-error correction (FEC) process may prove beneficial in improving communication systems over digital packet networks (\citet{zhai2007joint}, \citet{fouladi2018salsify}). Thus rather than relying on hand-crafted coding schemes that may require additional tuning in practice, we propose to \emph{learn} the appropriate bit allocations and coding mechanisms using a flexible deep learning framework.
\section{Neural Source and Channel Codes} 
Given the above considerations, we explore a learning-based approach to the JSCC problem. To do so, we consider a flexible \textit{class} of codes parameterized by a  neural network and jointly optimize for the encoding and decoding procedure. This approach is inspired by recent successes in training (discrete) latent variable generative models (\citet{neal1992connectionist}, \citet{rolfe2016discrete}, \citet{van2017neural}) as well as discrete representation learning in general (\citet{hu2017learning}, \citet{kaiser2018discrete}, \citet{roy2018theory}). We also explore the connection to different types of autoencoders in section \ref{sec:generative}.

\subsection{Coding Process}
Let $X,\hat{Y},Y,\hat{X}$ be random variables denoting the inputs, codewords, noisy codewords, and reconstructions respectively. We model their joint distribution  $p(x , \hat{y}, y, \hat{x})$
using the following graphical model
$X \rightarrow \hat{Y} \rightarrow Y \rightarrow \hat{X}$ as:
\begin{align}
&p(x , \hat{y}, y, \hat{x}) = \nonumber \\
&p_{\textrm{data}}(x) p_{\textrm{enc}}(\widehat{y}|x; \phi) p_{\textrm{channel}}(y|\widehat{y};\epsilon) p_{\textrm{dec}}(\hat{x}|y; \theta) \label{main:joint} 
\end{align}

In Eq. \ref{main:joint}, $p_{\textrm{data}}(x)$ denotes the distribution over inputs $\mathcal{X}$. It is not known explicitly in practice, and only accessible through samples. $p_{\text{channel}}(y|\widehat{y}; \epsilon)$ is the channel model, where we specifically focus on two widely used discrete channel models: (1) the binary erasure channel (BEC); and (2) the binary symmetric channel (BSC). 

The BEC erases each bit $\hat{y}_i$ into a corrupted symbol $?$ (e.g., $0 \rightarrow ?$) with some i.i.d probability $\epsilon$, but faithfully transmits the correct bit otherwise. Therefore, $\hat{Y}$ takes values in $\hat{\mathcal{Y}} = \{ 0, 1, ?\}^m$, and $p_{\text{BEC}}$ can be described by the following transition matrix for each bit:
\[
P_{\text{BEC}} = \begin{bmatrix}
1 - \epsilon & 0 \\
0 & 1 - \epsilon \\
\epsilon & \epsilon \\
\end{bmatrix}
\]  
where the first column denotes the transition dynamics for $y_i = 0$ and the second column is that of $y_i = 1$. Each element per column denotes the probability of the bit remaining the same, flipping to a different bit, or being erased. 

The BSC, on the other hand, independently flips each bit in the codeword with probability $\epsilon$ (e.g., $0 \rightarrow 1$). Therefore, $\hat{Y}$ takes values in $\hat{\mathcal{Y}} =\mathcal{Y} = \{ 0, 1 \}^m$ and
\[
p_{\text{BSC}}(y|\widehat{y}; \epsilon) = \\
\prod_{i=1}^m \epsilon^{y_i \oplus \widehat{y}_i} (1-\epsilon)^{y_i \oplus \widehat{y}_i \oplus 1}
\]
where $\oplus$ denotes addition modulo $2$ (i.e., an eXclusive OR). All experiments in this paper are conducted with the BSC channel. We note that the BSC is a more difficult communication channel to work with than the BEC (\citet{richardson2008modern}). 

A \emph{stochastic} encoder $q_{\text{enc}}(\hat{y}|x; \phi)$  generates a codeword $\hat{y}$ given an input $x$. Specifically, we model each bit $\hat{y}_i$ in the code with an independent Bernoulli random vector. We model the parameters of this Bernoulli with a neural network $f_\phi(\cdot)$ (an MLP or a CNN) parameterized by $\phi$. Denoting $\sigmoid(z)$ as the sigmoid function:
\[
q_{\text{enc}}(\hat{y}|x,\phi) = \prod_{i=1}^m \sigmoid(f_\phi(x_i))^{\hat{y}_i} (1 - \sigmoid(f_\phi(x_i)))^{(1-\hat{y}_i)}
\]
Similarly, we posit a probabilistic decoder $p_{\textrm{dec}}(\hat{x}|y; \theta)$ parameterized by an MLP/CNN that, given $y$, generates a decoded image $\hat{x}$. For real-valued images, we model each pixel as a factorized Gaussian with a fixed, isotropic covariance to yield: $\hat{x}|y \sim \mathcal{N}(f_\theta(y), \sigma^2I)$, where $f_\theta(\cdot)$ denotes the decoder network. For binary images, we use a Bernoulli decoder: $\hat{x}|y \sim \text{Bern}(f_\theta(y))$.
\section{Variational Learning of Neural Codes}

To learn an effective coding scheme, we \textbf{maximize the mutual information} between the input $X$ and the corresponding \emph{noisy} codeword $Y$ (\citet{barber2006kernelized}). That is, the code $\hat{y}$ should be robust to partial corruption; even its noisy instantiation $y$ should preserve as much information about the original input $x$ as possible (\citet{mackay2003information}). 

First, we note that we can analytically compute the encoding distribution $q_{\text{noisy\_enc}}(y|x;\epsilon, \phi)$ \emph{after} it has been perturbed by the channel by marginalizing over $\hat{y}$:
\[
q_{\text{noisy\_enc}}(y|x; \epsilon, \phi) = 
\sum_{\hat{y} \in \hat{\cal{Y}}} q_{\text{enc}}(\widehat{y}|x; \phi) p_{\textrm{channel}}(y|\widehat{y}; \epsilon)
\]
The BEC induces a 3-way Categorical distribution for $q_{\text{noisy\_enc}}$ where the third category refers to bit-erasure:
\[
\begin{split}
y \vert x \sim \textrm{Cat}(1 - \epsilon - \sigmoid(f_\phi(x)) + \sigmoid(f_\phi(x)) \cdot \epsilon, \\
\sigmoid(f_\phi(x)) - \sigmoid(f_\phi(x)) \cdot \epsilon, \epsilon)
\end{split}
\]
while for the BSC, it becomes:
\[
\begin{split}
q_{\text{noisy\_enc}}(y|x;\phi,\epsilon) = \prod_{i=1}^m
\left(\sigmoid(f_\phi(x_i)) - 2 \sigmoid(f_\phi(x_i)) \epsilon + \epsilon \right)^{y_i} \\
\left(1 - \sigmoid(f_\phi(x_i)) + 2 \sigmoid(f_\phi(x_i)) \epsilon - \epsilon \right)^{(1-y_i)}
\end{split}
\]

We note that although we have outlined two specific channel models, it is relatively straightforward for our framework to handle other types of channel noise.

Finally, we get the following optimization problem:
\begin{align}
    &\max_{\phi} I(X,Y;\phi,\epsilon) = H(X) - H(X|Y; \phi,\epsilon) \nonumber \\
      &= \mathbb{E}_{x \sim p_{\textrm{data}}(x)} \mathbb{E}_{y \sim q_{\text{noisy\_enc}}(y \mid x;\epsilon, \phi)} \left[\log p(x|y;\epsilon, \phi ) \right] + \textrm{const.}  \nonumber \\
     &\geq \mathbb{E}_{x \sim p_{\textrm{data}}(x)} \mathbb{E}_{y \sim q_{\text{noisy\_enc}}(y \mid x;\epsilon, \phi)} \left[\log p_{\textrm{dec}}(x|y; \theta) \right]+ \textrm{const.} \label{eq:opt_obj}
\end{align}
where $p(x|y;\epsilon, \phi )$ is the true (and intractable) posterior from Eq.~\ref{main:joint} and $p_{\textrm{dec}}(x|y; \theta)$ is an amortized variational approximation. The true posterior $p(x|y;\epsilon, \phi)$ --- the posterior probability over possible inputs $x$ given the received noisy codeword $y$ --- is the best possible decoder. However, it is also often intractable to evaluate and optimize.
We therefore use a \emph{tractable} variational approximation $p_{\textrm{dec}}(\hat{x}|y; \theta)$. Crucially, this variational approximation is \emph{amortized} (\citet{kingma2013auto}), and is the inference distribution that will actually be used for decoding at test time. Because of amortization, \emph{decoding is guaranteed to be efficient}, in contrast with existing error correcting codes, which typically involve NP-hard MPE inference queries.

Given any encoder ($\phi$), we can find the best amortized variational approximation by maximizing the lower bound (Eq.~\ref{eq:opt_obj}) as a function of $\theta$.
Therefore, the \modelname{}  training objective is given by:
\[
\max_{\theta, \phi} \mathbb{E}_{x \sim p_{\textrm{data}}(x)} \mathbb{E}_{y \sim q_{\text{noisy\_enc}}(y \mid x;\epsilon, \phi)}
\left[\log p_{\textrm{dec}}(x|y; \theta) \right]
\]
In practice, we approximate the expectation of the data distribution $p_{\text{data}}(x)$ with a finite dataset $\cal{D}$:
\begin{equation}
\label{main:objective}
    \max_{\theta, \phi} \sum_{x \in \cal{D}} \mathbb{E}_{y \sim q_{\text{noisy\_enc}}(y \mid x;\epsilon, \phi)}
    \left[\log p_{\textrm{dec}}(x|y;\theta) \right] \equiv \mathcal{L}(\phi, \theta; x, \epsilon)
\end{equation}
Thus we can jointly learn the encoding and decoding scheme by optimizing the parameters $\phi, \theta$. 
In this way, the encoder ($\phi$) is ``aware" of the computational limitations of the decoder ($\theta$), and is optimized accordingly~\cite{shu2018amortized}. However, a main learning challenge is that we are not able to backpropagate directly through the \emph{discrete} latent variable $y$; we elaborate upon this point further in Section \ref{sec:discreteopt}.

For the BSC, the solution will depend on the number of available bits $m$, the noise level $\epsilon$, and the structure of the input data $x$. For intuition, consider the case $m=0$. In this case, no information can be transmitted, and the best decoder will fit a single Gaussian to the data. When $m=1$ and $\epsilon=0$ (noiseless case), the model will learn a mixture of $2^m=2$ Gaussians. However, if $m=1$ and $\epsilon=0.5$, again no information can be transmitted, and the best decoder will fit a single Gaussian to the data. Adding noise forces the model to decide how it should effectively partition the data such that (1) similar items are grouped together in the same cluster; and (2) the clusters are "well-separated" (even when a few bits are flipped, they do not "cross over" to a different cluster that will be decoded incorrectly). Thus, from the perspective of unsupervised learning, \modelname{} attempts to learn \emph{robust binary representations} of the data.

\subsection{\modelname{} as a Generative Model}
\label{sec:generative}
The objective function in Eq.~\ref{main:objective} closely resembles those commonly used in generative modeling frameworks; in fact, \modelname{} can also be viewed as a generative model. In its simplest form, our model with a noiseless channel ($\text{i.e., } \epsilon = 0$), deterministic encoder, and deterministic decoder is identical to a traditional autoencoder (\citet{bengio2007scaling}). 
Once channel noise is present ($\epsilon > 0$) and the encoder/decoder become probabilistic, \modelname{} begins to more closely resemble other variational autoencoding frameworks. Specifically, \modelname{} is similar to Denoising Autoencoders (DAE) (\citet{vincent2008extracting}), except that it is explicitly trained for robustness to partial destruction of \textit{the latent space}, as opposed to the \emph{input} space. The stacked DAE \citep{vincent2010stacked} also denoises the latent representations, but does not explicitly inject noise into each layer.
\begin{figure*}[h]
\centering     
\subfigure[CIFAR10]{\label{fig:b}\includegraphics[width=0.33\textwidth]{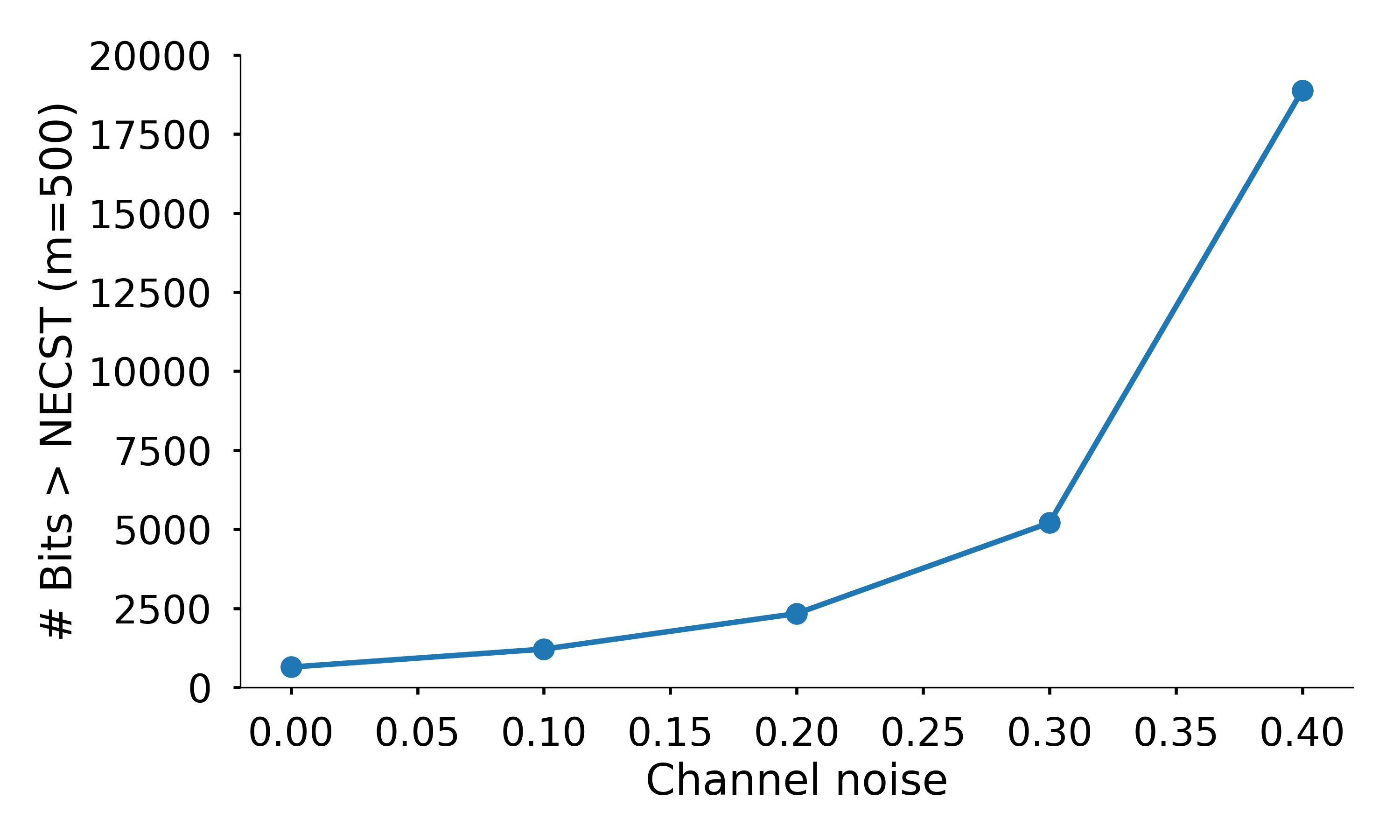}}
\subfigure[CelebA]{\label{fig:b}\includegraphics[width=0.33\textwidth]{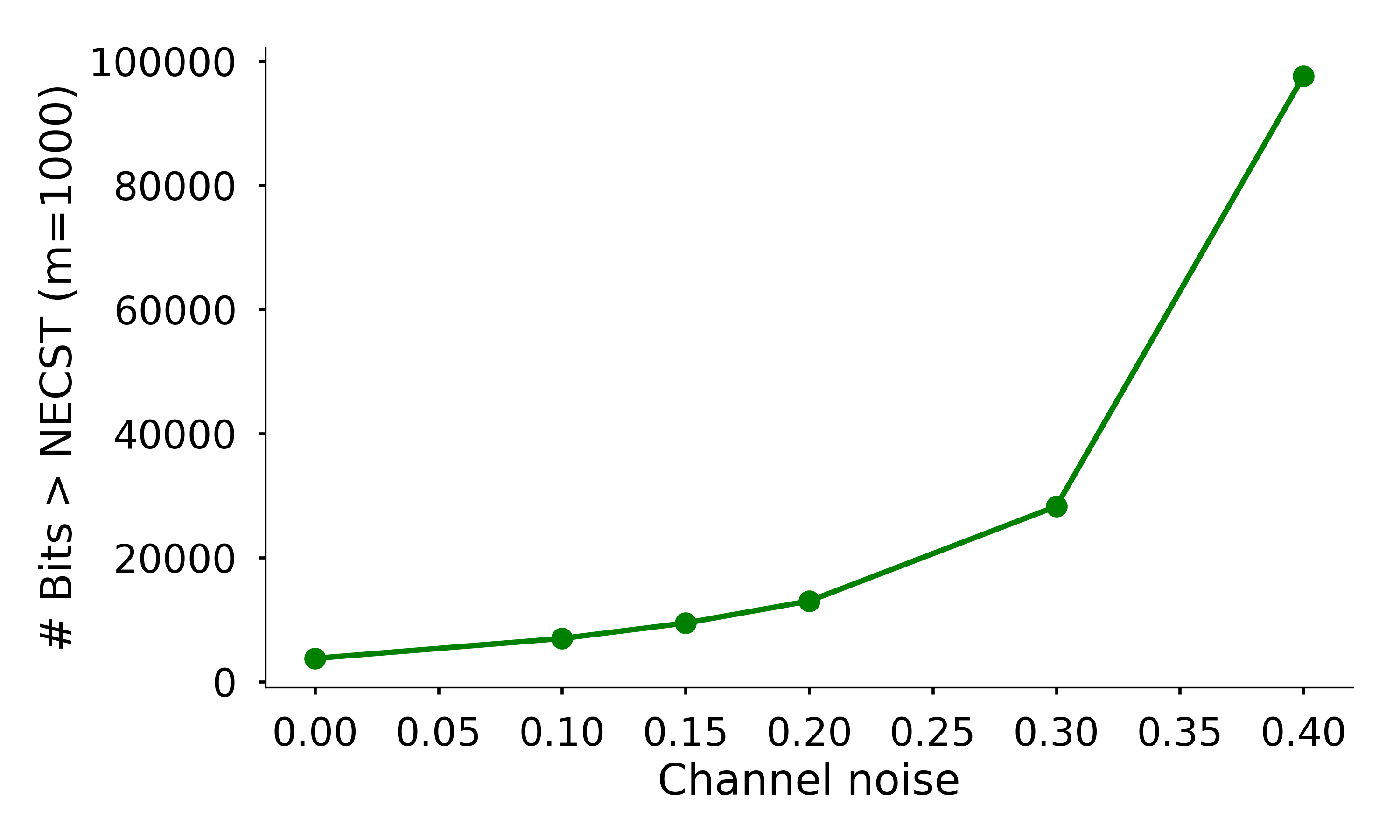}}
\subfigure[SVHN]{\label{fig:a}\includegraphics[width=0.33\textwidth]{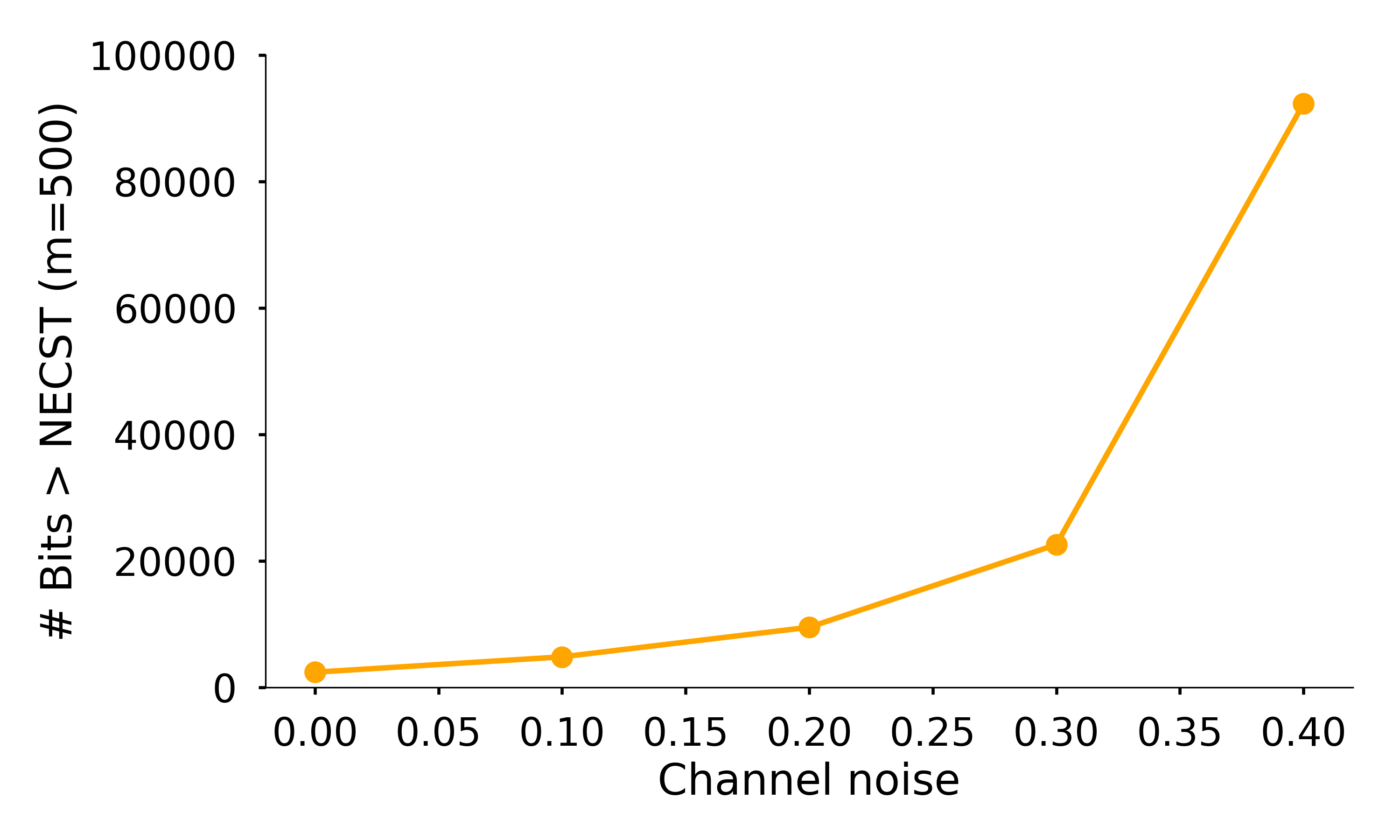}}
\caption{For each dataset, we fix the number of bits $m$ for \modelname{} and compute the resulting distortion. Then, we plot the \textit{additional} number of bits the WebP + ideal channel code system would need in theory to match \modelname{}'s distortion at different levels of channel noise. \modelname{}'s ability to jointly source and channel code yields much greater bitlength efficiency as compared to the separation scheme baseline, with the discrepancy becoming more dramatic as we increase the level of channel noise.}
\end{figure*}

\modelname{} is also related to the VAE, with two nuanced distinctions. While both models learn a joint distribution over the observations and latent codes, \modelname{}: (1) optimizes a variational lower bound to the \textit{mutual information} $I(X,Y)$ as opposed to the marginal log-likelihood $p(X)$, and (2) \textit{does not posit} a prior distribution $p(Y)$ over the latent variables. Their close relationship is evidenced by a line of work on rate-distortion optimization in the context of VAEs (\citet{balle2016end}, \citet{higgins2016beta}, \citet{zhao2017infovae}, \citet{alemi2017information},\citet{zhao2018lagrangian}), as well as other information-theoretic interpretations of the VAE's information preference (\citet{hinton1993keeping}, \citet{honkela2004variational}, \citet{chen2016variational}). Although \modelname{}'s method of representation learning is also related to that of the Deep Variational Information Bottleneck (VIB) \citep{alemi2016deep}, VIB \textit{minimizes} the MI between the observations and latents while maximizing the MI between the observations and targets in a supervised setting. Also, existing autoencoders are aimed at compression, while for sufficiently large $m$ \modelname{} will attempt to learn lossless compression with added redundancy for error correction. 

Finally, \modelname{} can be seen as a discrete version of the Uncertainty Autoencoder (UAE) (\citet{grover2018uae}). The two models share identical objective functions with two notable differences: For \modelname{}, (1) the latent codes $Y$ are discrete random variables, and (2) the noise model is no longer continuous. The special properties of the continuous UAE carry over directly to its discrete counterpart. \citet{grover2018uae} proved that under certain conditions, the UAE specifies an implicit generative model (\citet{diggle1984monte}, \citet{mohamed2016learning}). We restate their major theorem and extend their results here.

Starting from any data point $x^{(0)} \sim \mathcal{X}$, we define a Markov chain over $\cal{X} \times \cal{Y}$ with the following transitions:
\begin{equation}
y^{(t)} \sim q_{\text{noisy\_enc}}(y|x^{(t)};\phi,\epsilon)\ \ , \ \ x^{(t+1)} \sim p_{\textrm{dec}}(x|y^{(t)};\theta)
\label{eq:markovchain}
\end{equation}
\begin{theorem}
\label{thm:mc}
For a fixed value of $\epsilon>0$, let $\theta^\ast, \phi^\ast$ denote an optimal solution to the \modelname{} objective.
If there exists a $\phi$ such that $q_\phi(x \mid y; \epsilon) = p_{\theta^\ast}(x \mid y)$, 
then the Markov Chain (Eq.~\ref{eq:markovchain}) with parameters $\phi^\ast$ and $\theta^\ast$ is ergodic and its stationary distribution is given by 
$p_{\textrm{data}}(x)  q_{\text{noisy\_enc}}(y \mid x;\epsilon, \phi^\ast)$.
\end{theorem}
\begin{proof}

We follow the proof structure in \citet{grover2018uae} with minor adaptations. First, we note that the Markov chain defined in Eq.~\ref{eq:markovchain} is ergodic due to the BSC noise model. This is because the BSC defines a distribution over all possible (finite) configurations of the latent code $\mathcal{Y}$, and the Gaussian decoder posits a non-negative probability of transitioning to all possible reconstructions $\mathcal{X}$. Thus given any $(X,Y)$ and $(\mathcal{X}',\mathcal{Y}')$ such that the density $p(\mathcal{X},\mathcal{Y}) > 0$ and $p(\mathcal{X}',\mathcal{Y}') > 0$, the probability density of transitioning $p(\mathcal{X}',\mathcal{Y}'|\mathcal{X}, \mathcal{Y}) > 0$.

Next, we can rewrite the objective in Eq.~\ref{main:objective} as the following:
\begin{align*}
&\mathbb{E}_{p(x,y;\epsilon,\phi)} \left[ \log p_{\text{dec}}(x|y;\theta) \right] = \\
&\mathbb{E}_{q(y; \epsilon,\phi)}\left[\int q(x|y;\phi) \log p_{\text{dec}}(x|y;\theta) \right] =\\
&-H(X|Y;\phi) - \mathbb{E}_{q(y;\epsilon,\phi)}[KL(q(X|y; \phi) || p_{\textrm{dec}}(X|y;\theta)) ]
\end{align*}
Since the KL divergence is minimized when the two argument distributions are equal, we note that for the optimal value of $\theta = \theta^\ast$, if there exists a $\phi$ in the space of encoders being optimized that satisfies  $q_\phi(x \mid y; \epsilon) = p_{\theta^\ast}(x \mid y)$ then $\phi = \phi^\ast$.
Then, for any $\phi$, we note that the following Gibbs chain converges to $p(x,y)$ if the chain is ergodic:
\begin{equation}
y^{(t)} \sim q_{\text{noisy\_enc}}(y|x^{(t)};\phi,\epsilon)\ \ , \ \ x^{(t+1)} \sim p(x|y^{(t)};\theta^*)
\end{equation}
Substituting $p(x|y;\theta^\ast)$ for $p_{\textrm{dec}}(x|y;\theta)$ finishes the proof.
\end{proof}
Hence \modelname{} has an intractable likelihood but can generate samples from $p_{\textrm{data}}$ by running the Markov Chain above. Samples obtained by running the chain for a trained model, initialized from both random noise and test data, can be found in the Supplementary Material. Our results indicate that although we do not meet the theorem's conditions in practice, we are still able to learn a reasonable model.
\section{Experimental Results}
\label{exp_res}
  
We first review the optimization challenges of training discrete latent variable models and elaborate on our training procedure. Then to validate our work, we first assess \modelname{}'s compression and error correction capabilities against a combination of two widely-used compression (WebP, VAE) and channel coding (ideal code, LDPC (\citet{gallager1962low})) algorithms. We experiment on randomly generated length-100 bitstrings, MNIST (\citet{lecun1998mnist}), Omniglot (\citet{lake2015human}), Street View Housing Numbers (SVHN) (\citet{netzer2011reading}), CIFAR10 (\citet{krizhevsky2009learning}), and CelebA (\citet{liu2015faceattributes}) datasets to account for different image sizes and colors.
Next, we test the decoding speed of \modelname{}'s decoder and find that it performs upwards of \emph{an order of magnitude faster} than standard decoding algorithms based on iterative belief propagation (and two orders of magnitude on GPU).
Finally, we assess the quality of the latent codes after training, and examine interpolations in latent space as well as how well the learned features perform for downstream classification tasks.

\subsection{Optimization Procedure}
\label{sec:discreteopt}
Recall the \modelname{} objective in equation \ref{main:objective}.
While obtaining Monte Carlo-based gradient estimates with respect to $\theta$ is easy, gradient estimation with respect to $\phi$ is challenging because these parameters specify the Bernoulli random variable $q_{\textrm{noisy\_enc}}(y|x;\epsilon, \phi)$. The commonly used \textit{reparameterization trick} cannot be applied in this setting, as the discrete stochastic unit in the computation graph renders the overall network non-differentiable (\citet{schulman2015gradient}).

\begin{figure*}[h]
\begin{minipage}{0.6\textwidth}
\begin{tabular}{ccccccc}
\textbf{Binary MNIST} & 0.0 & 0.1 & 0.2 & 0.3 & 0.4 & 0.5\\ 
\hline
100-bit VAE+LDPC & 0.047 & 0.064 & 0.094 & 0.120 & 0.144 & 0.150\\
100-bit \modelname{} & \textbf{0.039} & \textbf{0.052} & \textbf{0.066} & \textbf{0.091} & \textbf{0.125} & \textbf{0.131} \\[1ex]

\textbf{Binary Omniglot} & 0.0 & 0.1 & 0.2 & 0.3 & 0.4 & 0.5\\
\hline
200-bit VAE+LDPC & 0.048 & 0.058 & 0.076 & 0.092 & 0.106 & 0.100\\
200-bit \modelname{} & \textbf{0.035} & \textbf{0.045} & \textbf{0.057} & \textbf{0.070} & \textbf{0.080} & \textbf{0.080} \\[1ex]

\textbf{Random bits} & 0.0 & 0.1 & 0.2 & 0.3 & 0.4 & 0.5\\ 
\hline
50-bit VAE+LDPC & 0.352 & 0.375 & 0.411 & \textbf{0.442} & \textbf{0.470} & 0.502\\
50-bit \modelname{} & \textbf{0.291} & \textbf{0.357} & \textbf{0.407} & 0.456 & 0.500 & \textbf{0.498}
\end{tabular}
    \end{minipage}
  \begin{minipage}{0.3\textwidth}
  \includegraphics[width=70mm]{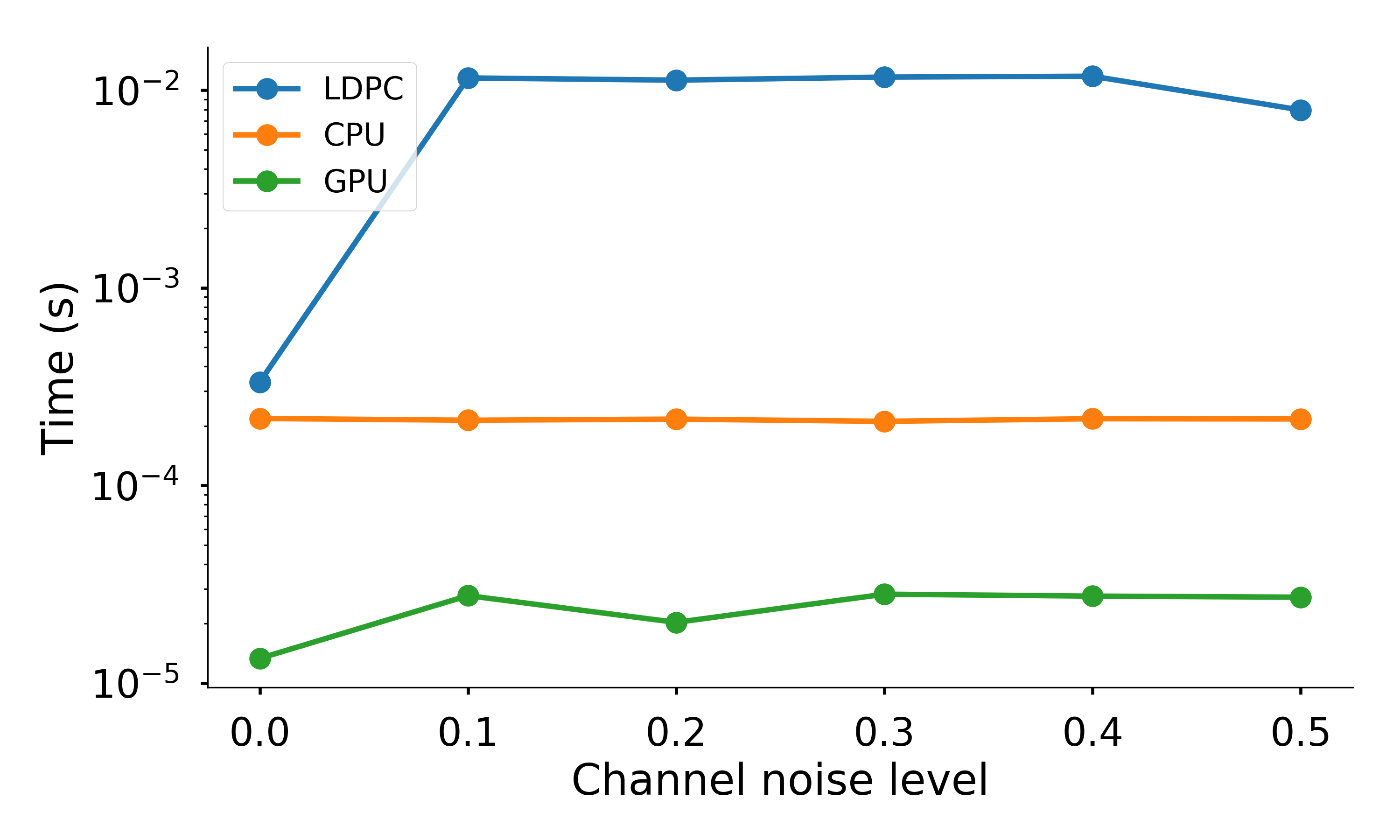}
\end{minipage}
    \caption{(Left) Reconstruction error of \modelname{} vs. VAE + rate-1/2 LDPC codes. We observe that NECST outperforms the baseline in almost all settings except for the random bits, which is expected due to its inability to leverage any underlying statistical structure. (Right) Average decoding times for \modelname{} vs. 50 iterations of LDPC decoding on Omniglot. One forward pass of NECST provides 2x orders of magnitude in speedup on GPU for decoding time when compared to a traditional algorithm based on iterative belief propagation.}
  \end{figure*}

An alternative is to use the \textit{score function estimator} in place of the gradient, as defined in the REINFORCE algorithm (\citet{williams1992simple}). However, this estimator suffers from high variance, and several others have explored different formulations and control variates to mitigate this issue (\citet{wang2013variance}, \citet{gu2015muprop}, \citet{ruiz2016generalized}, \citet{tucker2017rebar}, \citet{grathwohl2017backpropagation}, \citet{grover2018variational}). Others have proposed a continuous relaxation of the discrete random variables, as in the Gumbel-softmax (\citet{jang2016categorical}) and Concrete (\citet{maddison2016concrete}) distributions.

Empirically, we found that using that using a continuous relaxation of the discrete latent features led to worse performance at test time when the codes were forced to be discrete. Therefore, we used VIMCO (\citet{mnih2016variational}), a multi-sample variational lower bound objective for obtaining low-variance gradients. VIMCO constructs leave-one-out control variates using its samples, as opposed to the single-sample objective NVIL (\citet{mnih2014neural}) which requires learning additional baselines during training. Thus, we used the 5-sample VIMCO objective in subsequent experiments for the optimization procedure, leading us to our final multi-sample ($K=5$) objective:
\begin{equation}
\begin{split}
\label{main:vimco}
&\mathcal{L}^K(\phi, \theta; x, \epsilon) = \\
&\max_{\theta, \phi} \sum_{x \in \cal{D}} \mathbb{E}_{y^{1:K} \sim q_{\text{noisy\_enc}}(y \mid x;\epsilon, \phi)}\left[\log \frac{1}{K} \sum_{i=1}^K p_{\textrm{dec}}(x|y^{i};\theta) \right]
\end{split}
\end{equation}
  
\subsection{Fixed distortion: WebP + Ideal channel code}
In this experiment, we compare the performances of: (1) \modelname{} and (2) WebP + ideal channel code in terms of compression on 3 RGB datasets: CIFAR-10, CelebA, and binarized SVHN. We note that before the comparing our model with WebP, we remove the headers from the compressed images for a fair comparison.

Specifically, we fix the number of bits $m$ used by \modelname{} to source and channel code, and obtain the corresponding distortion levels (reconstruction errors) at various noise levels $\epsilon$. For fixed $m$, distortion will increase with $\epsilon$. Then, for each noise level we estimate the number of bits an alternate system using WebP and an ideal channel code --- the best that is theoretically possible --- would have to use to match \modelname{}'s distortion levels. As WebP is a variable-rate compression algorithm, we compute this quantity for each image, then average across the entire dataset. Assuming an ideal channel code implies that all messages will be transmitted across the noisy channel at the highest possible communication rate (i.e. the channel capacity); thus the resulting distortion will only be a function of the compression mechanism. We use the well-known channel capacity formula for the BSC channel $C = 1 - H_b(\epsilon)$ to compute this estimate, where $H_b(\cdot)$ denotes the binary entropy function.

We find that \modelname{} excels at compression. Figure 1 shows that in order to achieve the same level of distortion as \modelname{}, the WebP-ideal channel code system requires a much greater number bits \emph{across all noise levels}. In particular, we note that \modelname{} is slightly more effective than WebP at pure compression ($\epsilon=0$), and becomes significantly better at higher noise levels (e.g., for $\epsilon=0.4$, \modelname{} requires $20\times$ less bits).

\begin{figure*}[h]
\centering     
\subfigure[MNIST: 6 $\rightarrow$ 3 leads to an intermediate 5]{\label{fig:a}\includegraphics[width=80mm]{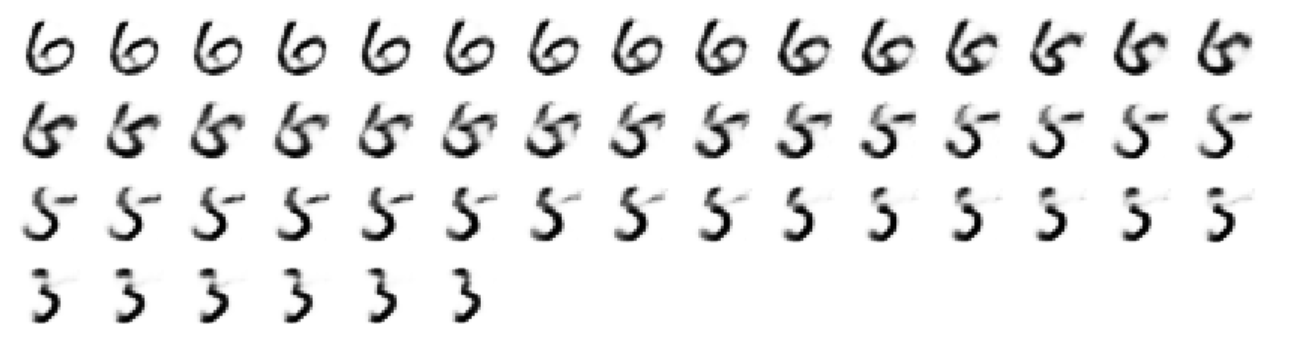}}
\subfigure[MNIST: 8 $\rightarrow$ 4 leads to an intermediate 2 and 9]{\label{fig:b}\includegraphics[width=85mm]{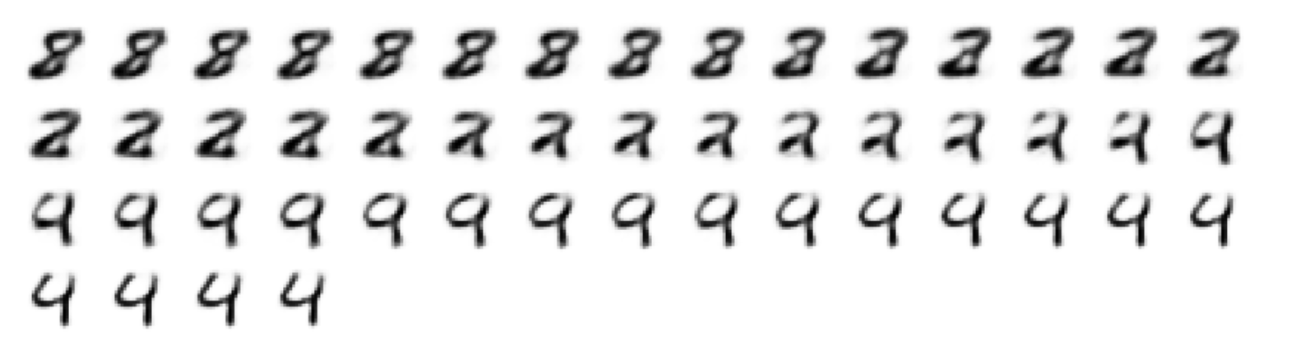}}
\caption{Latent space interpolation, where each image is obtained by randomly flipping one additional latent bit and passing through the decoder. We observe that the NECST model has indeed learned to encode redundancies in the compressed representations of the images, as it takes a significant number of bit flips to alter the original identity of the digit.}
\end{figure*}

\subsection{Fixed Rate: source coding + LDPC}
Next, we compare the performances of: (1) \modelname{} and (2) discrete VAE + LDPC in terms of distortion. Specifically, we fix the bit-length budget for both systems and evaluate whether learning compression and error-correction jointly (\modelname{}) improves reconstruction errors (distortion) for the same rate. We do not report a comparison against JPEG + LDPC as we were often unable to obtain valid JPEG files to compute reconstruction errors after imperfect LDPC decoding. However, as we have seen in the JPEG/WebP + ideal channel code experiment, we can reasonably conclude that the JPEG/WebP-LDPC system would have displayed worse performance, as \modelname{} was superior to a setup with the best channel code theoretically possible.

We experiment with a discrete VAE with a uniform prior over the latent codes for encoding on binarized MNIST and Omniglot, as these datasets are settings in which VAEs are known to perform favorably. We also include the random bits dataset to assess our channel coding capabilities, as the dataset lacks any statistical structure that can be exploited at source coding time. To fix the total number of bits that are transmitted through the channel, we double the length of the VAE encodings with a rate-1/2 LDPC channel code to match \modelname{}'s latent code dimension $m$. We found that fixing the number of LDPC bits and varying the number of bits in the VAE within a reasonable interval did not result in a significant difference. We note from Figure 2a that although the resulting distortion is similar, \modelname{} outperforms the VAE across all noise levels for image datasets. These results suggest that \modelname{}'s learned coding scheme performs at least as well as LDPCs, an industrial-strength, widely deployed class of error correcting codes.

\subsection{Decoding Time}
Another advantage of \modelname{} over traditional channel codes is that after training, the amortized decoder can very efficiently map the transmitted code into its best reconstruction at test time. Other sparse graph-based coding schemes such as LDPC, however, are more time-consuming because decoding involves an NP-hard optimization problem typically approximated with multiple iterations of belief propagation.

To compare the speed of our neural decoder to LDPC's belief propagation, we fixed the number of bits that were transmitted across the channel and averaged the total amount of time taken for decoding across ten runs. The results are shown below in Figure 2b for the statically binarized Omniglot dataset. On CPU, \modelname{}'s decoder averages an order of magnitude faster than the LDPC decoder running for 50 iterations, which is based on an optimized C implementation. On GPU, \modelname{} displays \emph{two orders of magnitude in speedup}. We also conducted the same experiment without batching, and found that NECST still remains an order of magnitude faster than the LDPC decoder.

\section{Robust Representation Learning}

\begin{table*}[t!]
\small
\caption{Classification accuracy on MNIST/noisy MNIST using 100-bit features from \modelname{}. 
}
\begin{center}
\begin{tabular}{c|c|c|c|c|c|c|c|c}
\hline
Noise $\epsilon$ & KNN & DT & RF & MLP & AdaB & NB & QDA & SVM\\ 
\hline
0 & \textbf{0.95}/\textbf{0.86} & 0.65/0.54 & 0.71/0.59 & 0.93/0.87 & 0.72/0.65 & 0.75/0.65 & 0.56/0.28 & 0.88/0.81\\ 
\hline
0.1 & 0.95/0.86 & 0.65/0.59 & 0.74/0.65 & \textbf{0.934}/0.88 & 0.74/0.72 & 0.83/0.77 & \textbf{0.94}/\textbf{0.90} & 0.92/0.84\\ 
\hline
0.2 & 0.94/0.86 & \textbf{0.78}/\textbf{0.69} & \textbf{0.81}/\textbf{0.76} & 0.93/\textbf{0.89} & \textbf{0.78}/\textbf{0.80} & \textbf{0.87}/\textbf{0.81} & 0.93/0.90 & \textbf{0.93}/\textbf{0.86}\\ 
\hline
\end{tabular}
\end{center}
\end{table*}
In addition to its ability to source and channel code effectively, \modelname{} is also an implicit generative model that yields robust and interpretable latent representations and realistic samples from the underlying data distribution. Specifically, in light of Theorem \ref{thm:mc}, we can think of $q_{\text{enc}}(\widehat{y}|x; \phi)$ as mapping images $x$ to latent representations $\widehat{y}$.

\subsection{Latent space interpolation}
To assess whether the model has: (1) injected redundancies into the learned codes and (2) learned interesting features, we interpolate between different data points in latent space and qualitatively observe whether the model captures semantically meaningful variations in the data. We select two test points to be the start and end, sequentially flip one bit at a time in the latent code, and pass the altered code through the decoder to observe how the reconstruction changes. In Figure 3, we show two illustrative examples from the MNIST digits. From the starting digit, each bit-flip slowly alters characteristic features such as rotation, thickness, and stroke style until the digit is reconstructed to something else entirely. We note that because \modelname{} is trained to encode redundancies, we do not observe a drastic change per flip. Rather, it takes a significant level of corruption for the decoder to interpret the original digit as another. Also, due to the i.i.d. nature of the channel noise model, we observe that the sequence at which the bits are flipped do not have a significant effect on the reconstructions. Additional interpolations may be found in the Supplementary Material.

\subsection{Downstream classification}
To demonstrate that the latent representations learned by \modelname{} are useful for downstream classification tasks, we extract the binary features and train eight different classification algorithms: $k$-nearest neighbors (KNN), decision trees (DT), random forests (RF), multilayer perceptron (MLP), AdaBoost (AdaB), Quadratic Discriminant Analysis (QDA), and support vector machines (SVM). As shown on the leftmost numbers in Table 1, the simple classifiers perform reasonably well in the digit classification task for MNIST ($\epsilon=0$). We observe that with the exception of KNN, all other classifiers achieve higher levels of accuracy when using features that were trained with simulated channel noise ($\epsilon = 0.1, 0.2$). 

To further test the hypothesis that training with simulated channel noise yields \textit{more robust} latent features, we freeze the pre-trained \modelname{} model and evaluate classification accuracy using a "noisy" MNIST dataset. We synthesized noisy MNIST by adding $\epsilon \sim \mathcal{N}(0, 0.5)$ noise to all pixel values in MNIST, and ran the same experiment as above. As shown in the rightmost numbers of Table 1, we again observe that most algorithms show improved performance with added channel noise. Intuitively, when the latent codes are corrupted by the channel, the codes will be "better separated" in latent space such that the model will still be able to reconstruct accurately despite the added noise. Thus \modelname{} can be seen as a "denoising autoencoder"-style method for learning more robust latent features, with the twist that the noise is injected into the latent space.

\section{Related work}
There has been a recent surge of work applying deep learning and generative modeling techniques to lossy image compression, many of which compare favorably against industry standards such as JPEG, JPEG2000, and WebP (\citet{toderici2015variable}, \citet{balle2016end}, \citet{toderici2017full}). \citet{theis2017lossy} use compressive autoencoders that learn the optimal number of bits to represent images based on their pixel frequencies. \citet{balle2018variational} use a variational autoencoder (VAE) (\citet{kingma2013auto} \citet{rezende2014stochastic}) with a learnable scale hyperprior to capture the image's partition structure as side information for more efficient compression. \citet{santurkar2017generative} use adversarial training (\citet{goodfellow2014generative}) to learn neural codecs for compressing images and videos using DCGAN-style ConvNets (\citet{radford2015unsupervised}). Yet these methods focus on source coding only, and do not consider the setting where the compression must be robust to channel noise.

In a similar vein, there has been growing interest on leveraging these deep learning systems to sidestep the use of hand-designed codes. Several lines of work train neural decoders based on known coding schemes, sometimes learning more general decoding algorithms than before (\citet{nachmani2016learning}, \citet{gruber2017deep}, \citet{cammerer2017scaling}, \citet{dorner2017deep}). (\citet{kim2018deepcode}, \citet{kim2018communication}) parameterize sequential codes with recurrent neural network (RNN) architectures that achieve comparable performance to capacity-approaching codes over the additive white noise Gaussian (AWGN) and bursty noise channels. However, source coding is out of the scope for these methods that focus on learning good channel codes.

The problem of end-to-end transmission of structured data, on the other hand, is less well-studied. \citet{zarcone2018joint} utilize an autoencoding scheme for data storage. \citet{farsad2018deep} use RNNs to communicate text over a BEC and are able to preserve the words' semantics. The most similar to our work is that of \citet{bourtsoulatze2018deep}, who use autoencoders for transmitting images over the AWGN and slow Rayleigh fading channels, which are continuous. We provide a holistic treatment of various discrete noise models and show how \modelname{} can also be used for unsupervised learning. While there is a rich body of work on using information maximization for representation learning (\citet{chen2016infogan}, \citet{hu2017learning}, \citet{hjelm2018learning}), these methods do not incorporate the addition of discrete latent noise in their training procedure.
\section{Conclusion}
\label{disc}
We described how \modelname{} can be used to learn an efficient joint source-channel coding scheme by simulating a noisy channel during training. We showed that the model: (1) is competitive against a combination of WebP and LDPC codes on a wide range of datasets, (2) learns an extremely fast neural decoder through amortized inference, and (3) learns a latent code that is not only robust to corruption, but also useful for downstream tasks such as classification.

One limitation of \modelname{} is the need to train the model separately for different code-lengths and datasets; in its current form, it can only handle fixed-length codes. Extending the model to streaming or adaptive learning scenarios that allow for learning variable-length codes is an exciting direction for future work. Another direction would be to analyze the characteristics and properties of the learned latent codes under different discrete channel models. We provide reference implementations in Tensorflow \citep{abadi2016tensorflow}, and the codebase for this work is open-sourced at \texttt{https://github.com/ermongroup/necst}.
\section*{Acknowledgements}
\label{acks}
We are thankful to Neal Jean, Daniel Levy, Rui Shu, and Jiaming Song for insightful discussions and feedback on early drafts. KC is supported by the NSF GRFP and Stanford Graduate Fellowship, and AG is supported by the MSR Ph.D. fellowship, Stanford Data Science scholarship, and Lieberman fellowship. This research was funded by NSF (\#1651565, \#1522054, \#1733686), ONR (N00014-19-1-2145), AFOSR (FA9550-19-1-0024), and Amazon AWS. 

\bibliography{references}
\bibliographystyle{icml2019}

\raggedbottom
\pagebreak
\pagebreak
\section*{Supplementary Material}
\renewcommand\thesection{\Alph{section}}
\setcounter{section}{0}

\section{\modelname{} architecture and hyperparameters}
\subsection{MNIST}
For MNIST, we used the static binarized version as provided in (\citet{burda2015importance}) with train/validation/test splits of 50K/10K/10K respectively.
\begin{itemize}
    \item encoder: MLP with 1 hidden layer (500 hidden units), ReLU activations
    \item decoder: 2-layer MLP with 500 hidden units each, ReLU activations. The final output layer has a sigmoid activation for learning the parameters of $p_{\textrm{noisy\_enc}}(y|x;\phi,\epsilon)$
    \item n\_bits: 100
    \item n\_epochs: 200
    \item batch size: 100
    \item L2 regularization penalty of encoder weights: 0.001
    \item Adam optimizer with lr=0.001
\end{itemize}

\subsection{Omniglot}
We statically binarize the Omniglot dataset by rounding values above 0.5 to 1, and those below to 0. The architecture is the same as that of the MNIST experiment.
\begin{itemize}
    \item n\_bits: 200
    \item n\_epochs: 500
    \item batch size: 100
    \item L2 regularization penalty of encoder weights: 0.001
    \item Adam optimizer with lr=0.001
\end{itemize}

\subsection{Random bits}
We randomly generated length-100 bitstrings by drawing from a $\textrm{Bern}(0.5)$ distribution for each entry in the bitstring. The train/validation/test splits are: 5K/1K/1K. The architecture is the same as that of the MNIST experiment.
\begin{itemize}
    \item n\_bits: 50
    \item n\_epochs: 200
    \item batch size: 100
    \item L2 regularization penalty of encoder weights: 0.001
    \item Adam optimizer with lr=0.001
\end{itemize}

\subsection{SVHN}
For SVHN, we collapse the "easier" additional examples with the more difficult training set, and randomly partition 10K of the roughly 600K dataset into a validation set.
\begin{itemize}
    \item encoder: CNN with 3 convolutional layers + fc layer, ReLU activations
    \item decoder: CNN with 4 deconvolutional layers, ReLU activations. 
    \item n\_bits: 500
    \item n\_epochs: 500
    \item batch size: 100
    \item L2 regularization penalty of encoder weights: 0.001
    \item Adam optimizer with lr=0.001
\end{itemize}
The CNN architecture for the encoder is as follows:
\begin{enumerate}
    \item conv1 = n\_filters=128, kernel\_size=2, strides=2, padding="VALID"
    \item conv2 = n\_filters=256, kernel\_size=2, strides=2, padding="VALID"
    \item conv3 = n\_filters=512, kernel\_size=2, strides=2, padding="VALID"
    \item fc = 4*4*512 $\rightarrow$ n\_bits, no activation
\end{enumerate}
The decoder architecture follows the reverse, but with a final deconvolution layer as: n\_filters=3, kernel\_size=1, strides=1, padding="VALID", activation=ReLU.

\subsection{CIFAR10}
We split the CIFAR10 dataset into train/validation/test splits.
\begin{itemize}
    \item encoder: CNN with 3 convolutional layers + fc layer, ReLU activations
    \item decoder: CNN with 4 deconvolutional layers, ReLU activations. 
    \item n\_bits: 500
    \item n\_epochs: 500
    \item batch size: 100
    \item L2 regularization penalty of encoder weights: 0.001
    \item Adam optimizer with lr=0.001
\end{itemize}
The CNN architecture for the encoder is as follows:
\begin{enumerate}
    \item conv1 = n\_filters=64, kernel\_size=3, padding="SAME"
    \item conv2 = n\_filters=32, kernel\_size=3, padding="SAME"
    \item conv3 = n\_filters=16, kernel\_size=3, padding="SAME"
    \item fc = 4*4*16 $\rightarrow$ n\_bits, no activation
\end{enumerate}
Each convolution is followed by batch normalization, a ReLU nonlinearity, and 2D max pooling. 
The decoder architecture follows the reverse, where each deconvolution is followed by batch normalization, a ReLU nonlinearity, and a 2D upsampling procedure. Then, there is a final deconvolution layer as: n\_filters=3, kernel\_size=3, padding="SAME" and one last batch normalization before a final sigmoid nonlinearity.

\subsection{CelebA}
We use the CelebA dataset with standard train/validation/test splits with minor preprocessing. First, we align and crop each image to focus on the face, resizing the image to be $(64, 64, 3)$.
\begin{itemize}
    \item encoder: CNN with 5 convolutional layers + fc layer, ELU activations
    \item decoder: CNN with 5 deconvolutional layers, ELU activations. 
    \item n\_bits: 1000
    \item n\_epochs: 500
    \item batch size: 100
    \item L2 regularization penalty of encoder weights: 0.001
    \item Adam optimizer with lr=0.0001
\end{itemize}
The CNN architecture for the encoder is as follows:
\begin{enumerate}
    \item conv1 = n\_filters=32, kernel\_size=4, strides=2, padding="SAME"
    \item conv2 = n\_filters=32, kernel\_size=4, strides=2, padding="SAME"
    \item conv3 = n\_filters=64, kernel\_size=4, strides=2, padding="SAME"
    \item conv4 = n\_filters=64, kernel\_size=4, strides=2, padding="SAME"
    \item conv5 = n\_filters=256, kernel\_size=4, strides=2, padding="VALID"
    \item fc = 256 $\rightarrow$ n\_bits, no activation
\end{enumerate}
The decoder architecture follows the reverse, but without the final fully connected layer and the last deconvolutional layer as: n\_filters=3, kernel\_size=4, strides=2, padding="SAME", activation=sigmoid.

\section{Additional experimental details and results}
\subsection{WebP/JPEG-Ideal Channel Code System}
For the BSC channel, we can compute the theoretical channel capacity with the formula $C = 1 - H_b(\epsilon)$, where $\epsilon$ denotes the bit-flip probability of the channel and $H_b$ denotes the binary entropy function. 
Note that the communication rate of $C$ is achievable in the asymptotic scenario of infinitely long messages; in the finite bit-length regime, particularly in the case of short blocklengths, the highest achievable rate will be much lower. 

For each image, we first obtain the target distortion $d$ per channel noise level by using a fixed bit-length budget with \modelname{}. Next we use the JPEG compressor to encode the image at the distortion level $d$. The resulting size of the compressed image $f(d)$ is used to get an estimate $f(d)/C$ for the number of bits used for the image representation in the ideal channel code scenario. While measuring the compressed image size $f(d)$, we ignore the header size of the JPEG image, as the header is similar for images from the same dataset. 

The plots compare $f(d)/C$ with $m$, the fixed bit-length budget for \modelname{}.

\begin{figure*}[h]
\centering     
\subfigure[BinaryMNIST]{\label{fig:a}\includegraphics[width=.3\textwidth]{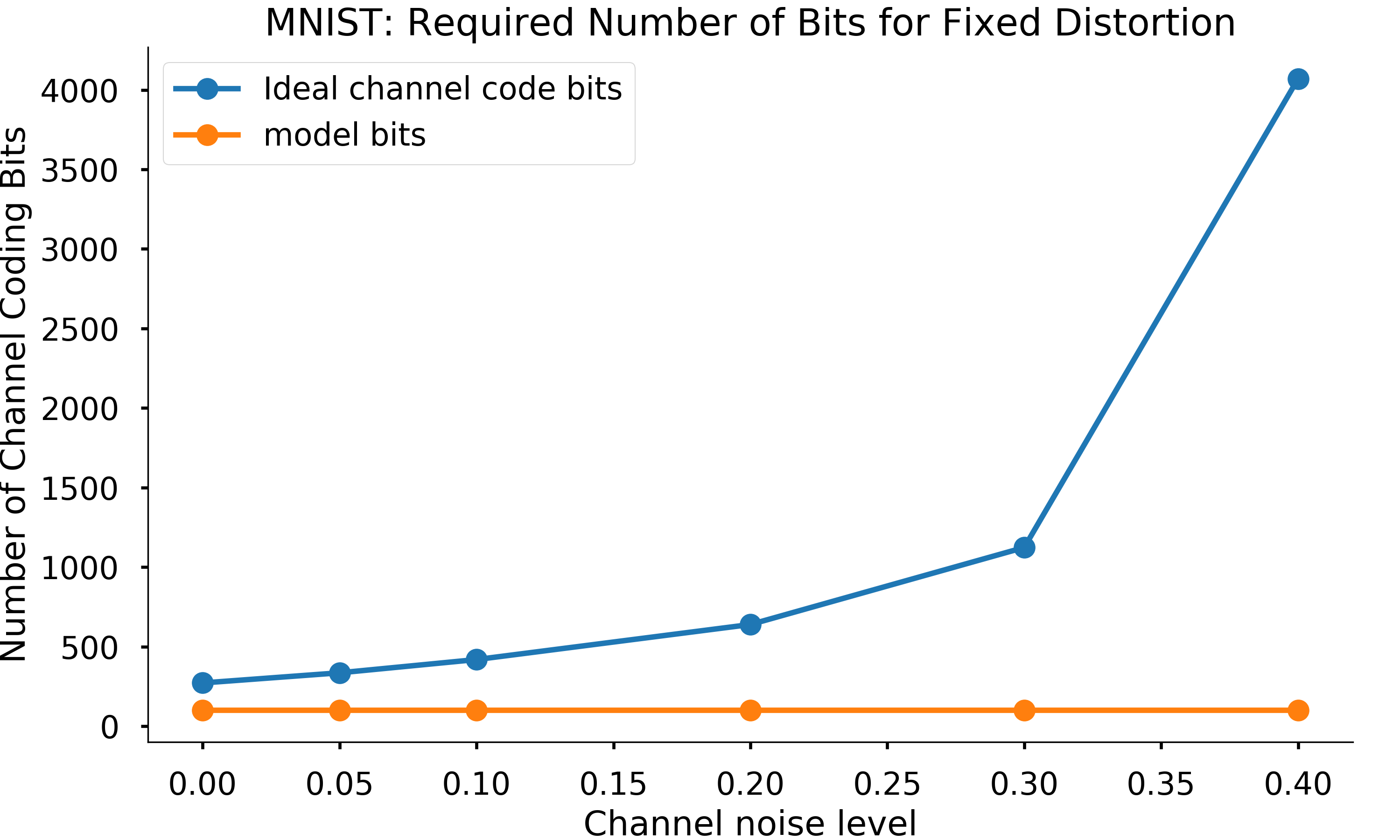}}
\subfigure[Omniglot]{\label{fig:b}\includegraphics[width=.3\textwidth]{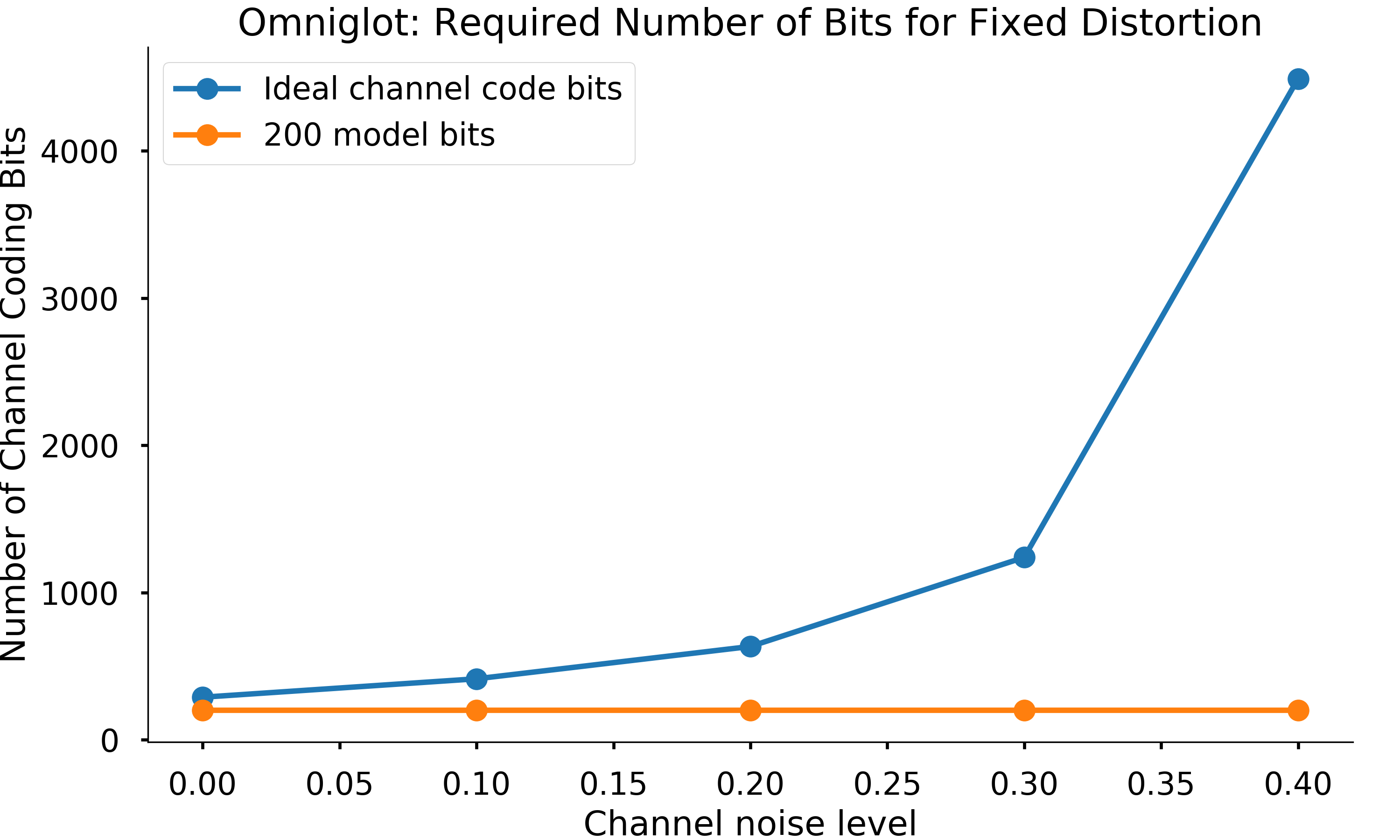}}
\subfigure[random]{\label{fig:c}\includegraphics[width=.3\textwidth]{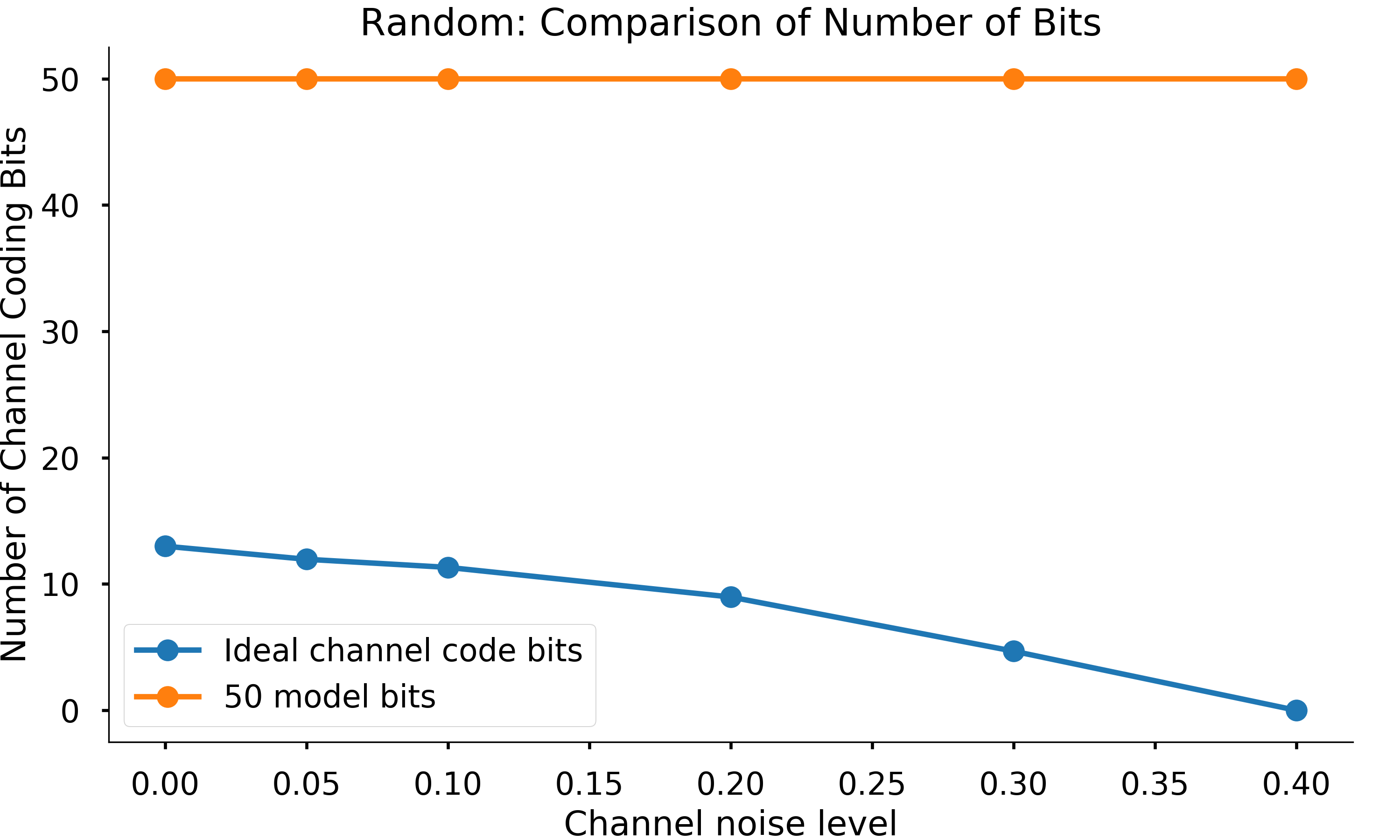}}
\caption{Theoretical $m$ required for JPEG + ideal channel code to match \modelname{}'s distortion.}
\end{figure*}

\subsection{Fixed Rate: JPEG-LDPC system}
We first elaborate on the usage of the LDPC software. We use an optimized C implementation of LDPC codes to run our decoding time experiments:\\
(\texttt{http://radfordneal.github.io/LDPC-codes/}). 
The procedure is as follows:
\begin{itemize}
    \item \texttt{make-pchk} to create a parity check matrix for a regular LDPC code with three 1's per column, eliminating cycles of length 4 (default setting). The number of parity check bits is: \texttt{total number of bits allowed} - \texttt{length of codeword}.
    \item \texttt{make-gen} to create the generator matrix from the parity check matrix. We use the default \texttt{dense} setting to create a dense representation.
    \item \texttt{encode} to encode the source bits into the LDPC encoding
    \item \texttt{transmit} to transmit the LDPC code through a \texttt{bsc} channel, with the appropriate level of channel noise specified (e.g. 0.1)
    \item \texttt{extract} to obtain the actual decoded bits, or \texttt{decode} to directly obtain the bit errors from the source to the decoding.
\end{itemize}
 
LDPC-based channel codes require larger blocklengths to be effective. To perform an end-to-end experiment with the JPEG compression and LDPC channel codes, we form the input by concatenating multiple blocks of images together into a grid-like image. In the first step, the fixed rate of $m$ is scaled by the total number of images combined, and this quantity is used to estimate the target $f(d)$ to which we compress the concatenated images. In the second step, the compressed concatenated image is coded together by the LDPC code into a bit-string, so as to correct for any errors due to channel noise. 

Finally, we decode the corrupted bit-string using the LDPC decoder. The plots compare the resulting distortion of the compressed concatenated block of images with the average distortion on compressing the images individually using \modelname{}. Note that, the experiment gives a slight disadvantage to \modelname{} as it compresses every image individually, while JPEG compresses multiple images together. 

We report the average distortions for sampled images from the test set.

Unfortunately, we were unable to run the end-to-end for some scenarios and samples due to errors in the decoding (LDPC decoding, invalid JPEGs etc.).

\subsection{VAE-LDPC system}
For the VAE-LDPC system, we place a uniform prior over all the possible latent codes and compute the KL penalty term between this prior $p(y)$ and the random variable $q_{\textrm{noisy\_enc}}(y|x;\phi, \epsilon)$. The learning rate, batch size, and choice of architecture are data-dependent and fixed to match those of \modelname{} as outlined in Section C. However, we use half the number of bits as allotted for \modelname{} so that during LDPC channel coding, we can double the codeword length in order to match the rate of our model.

\subsection{Interpolation in Latent Space}
We show results from latent space interpolation for two additional datasets: SVHN and celebA. We used 500 bits for SVHN and 1000 bits for celebA across channel noise levels of [0.0, 0.1, 0.2, 0.3, 0.4, 0.5].
\begin{figure*}[h]
\centering
\includegraphics[width=130mm]{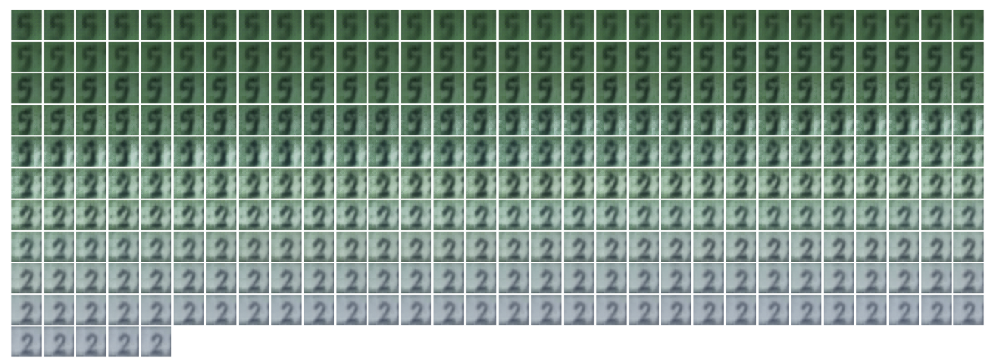}
\caption{Latent space interpolation of 5 $\rightarrow$ 2 for SVHN, 500 bits at noise=0.1}
\end{figure*}
\begin{figure*}[h]
\centering
\includegraphics[width=130mm]{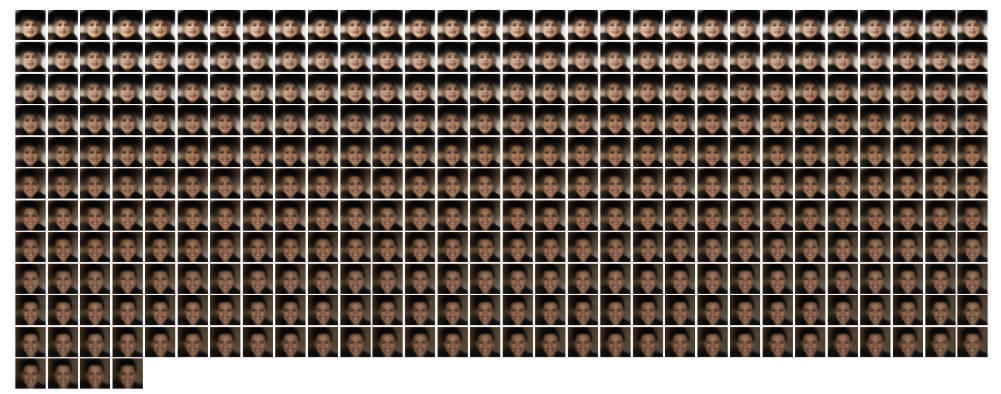}
\caption{Latent space interpolation for celebA, 1000 bits at noise=0.1}
\end{figure*}

\subsection{Markov chain image generation}
We observe that we can generate diverse samples from the data distribution after initializing the chain with both: (1) examples from the test set and (2) random Gaussian noise $x_0 \sim \mathcal{N}(0, 0.5)$. 
\begin{figure*}[h]
\centering     

\subfigure[MNIST, initialized from random noise]{\label{fig:a}\includegraphics[width=43mm]{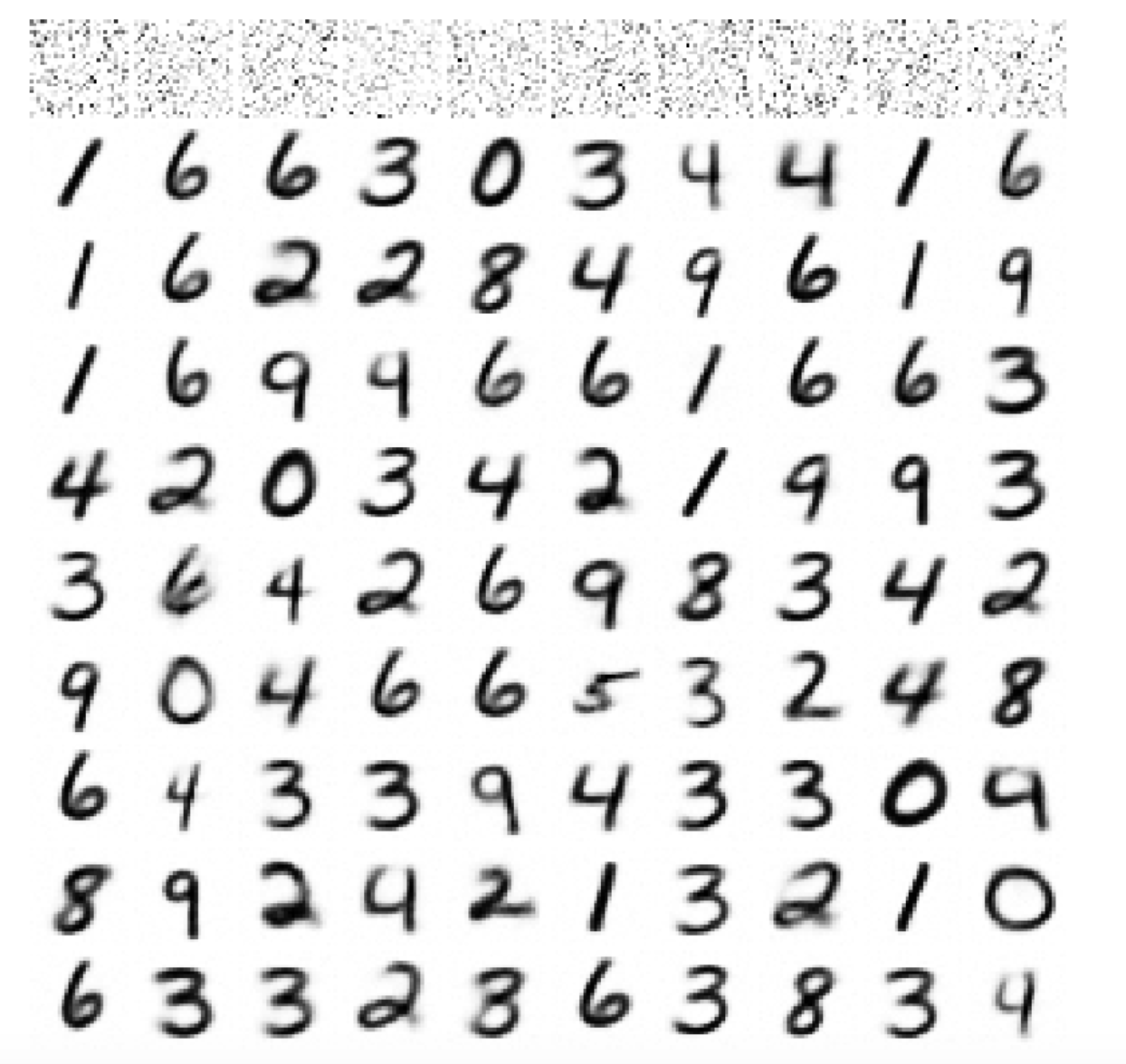}}
\subfigure[celebA, initialized from data]{\label{fig:b}\includegraphics[width=43mm]{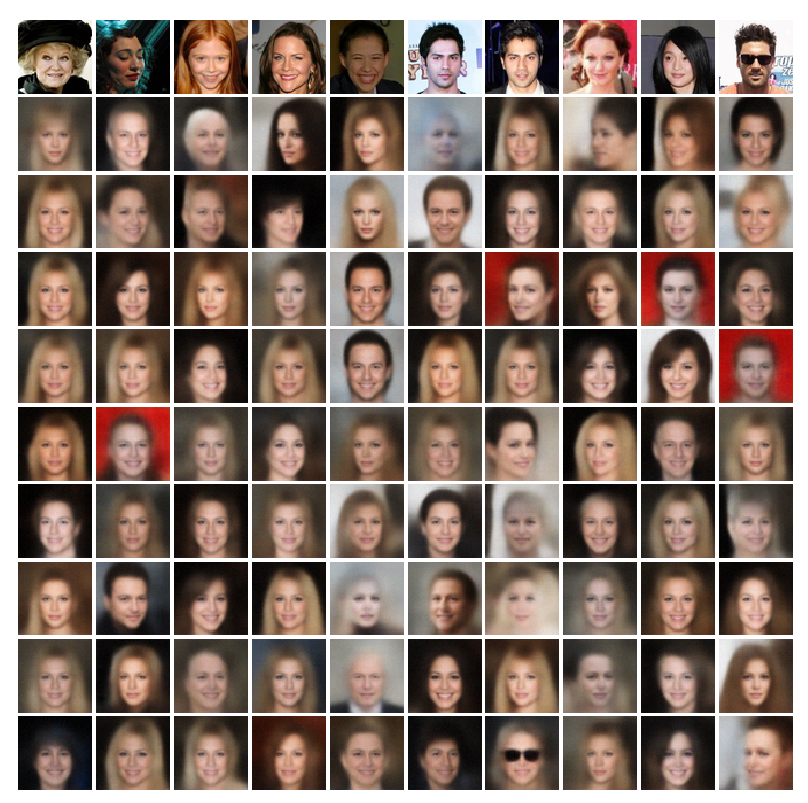}}
\subfigure[SVHN, initialized from data]{\label{fig:b}\includegraphics[width=43mm]{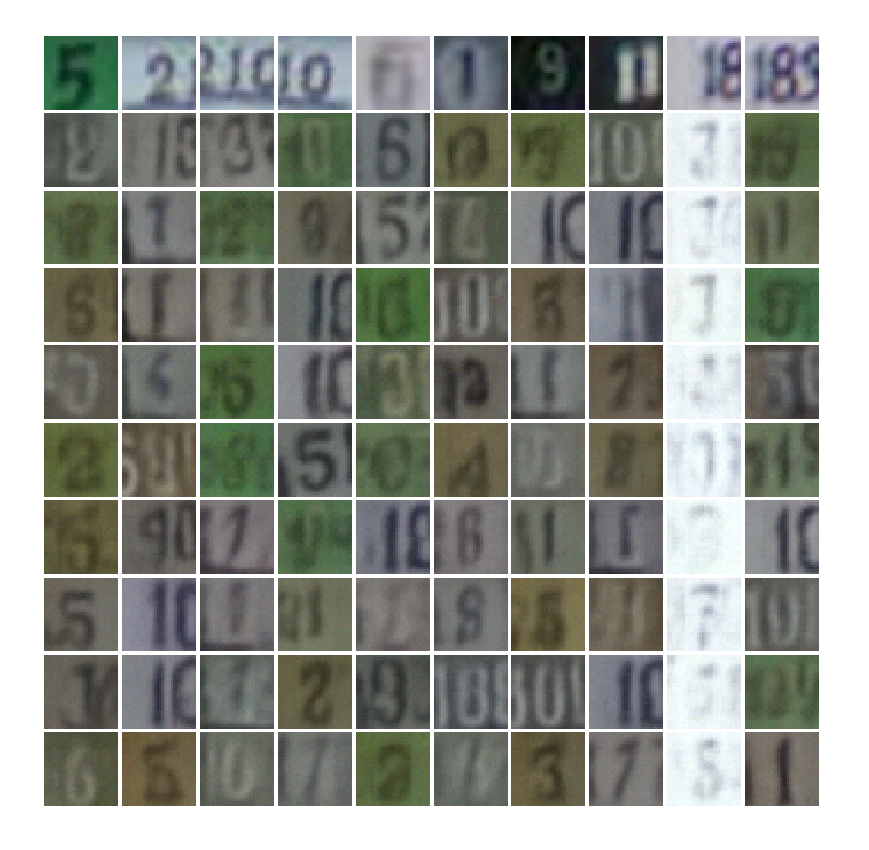}}
\caption{Markov chain image generation after 9000 timesteps, sampled per 1000 steps}
\end{figure*}

\subsection{Downstream classification}
Following the setup of (\citet{grover2018uae}), we used standard implementations in \texttt{sklearn} with default parameters for all 8 classifiers with the following exceptions:
\begin{enumerate}
    \item KNN: \texttt{n\_neighbors=3}
    \item DF:  \texttt{max\_depth=5}
    \item RF: \texttt{max\_depth=5, n\_estimators=10, max\_features=1}
    \item MLP:  \texttt{alpha=1}
    \item SVC: \texttt{kernel=linear, C=0.025}
\end{enumerate}

\end{document}